\documentclass[CHICAGO,twocolumn]{USG}
\usepackage{anyfontsize}
\usepackage{etoolbox}

\graphicspath{{./images/}}

\articletype{RESEARCH ARTICLE}

\journal{Preprint}
\copyyear{2026}
\startpage{1}

\makeatletter
\historydates{}
\patchcmd{\@maketitle}{\includegraphics[width=26.5mm]{Wiley_logo}}{}{}{}
\def\@@articletype{%
  \rule[-4\@p@t]{4\@p@t}{15\@p@t}%
  {\arttypefont\bf\uppercase{\letterspace to 1.15\naturalwidth{\@articletype}}}%
}
\def\oddfoot@titlepage@info{%
  \vbox{%
    \hsize\textwidth
    \vskip-3\@p@t
    \noindent\rule{\textwidth}{.5\@p@t}\par
    \vskip2\@p@t
    {\fontsize{7}{8}\selectfont{\it\@journal}, \@copyyear\hfill{\pagenumfont\thepage\ of \@LastPg}}%
    \vskip10\@p@t
  }%
}
\makeatother

\begin{document}

\title{StochIKM: Automated High-Fidelity Reconstruction of Individual Patient Data from Published Kaplan--Meier Curves}

\author[1]{Liangze Ke}
\author[1]{Meixuan Song}
\author[1]{Yujie Zhang}
\author[1]{Daniel Zelterman}
\author[2]{Wei Wei}
\author[3]{Guannan Gong}

\authormark{KE \textsc{et al.}}

\address[1]{\orgdiv{Department of Biostatistics, }\orgname{School of Public Health, Yale University, }%
\orgaddress{\state{New Haven, CT 06520, }\country{USA}}}

\address[2]{\orgdiv{Section of Medical Oncology, }\orgname{Yale School of Medicine and Yale Cancer Center, Yale University, }%
\orgaddress{\state{New Haven, CT 06520, }\country{USA}}}

\address[3]{\orgdiv{Department of Medicine, }\orgname{Yale University School of Medicine and Yale Cancer Center, }%
\orgaddress{\state{New Haven, CT 06520, }\country{USA}}}

\corres{Wei Wei (\email{wei.wei@yale.edu}) and Guannan Gong (\email{guannan.gong@yale.edu})}
\abstract[ABSTRACT]{Individual patient data are the gold standard for secondary analysis of
time-to-event evidence, but they are rarely released, and reconstruction from a
published Kaplan--Meier figure could help. Existing methods all claim
high fidelity, each judged by whichever agreements it happens to achieve, and
each returns a single dataset although the map from figure to data is not
bijective. We define fidelity explicitly: a dataset is admissible when it
reproduces the risk table exactly, follows the digitized curve to within the
image resolution, and respects the printed median, interval and event count. We propose StochIKM, a stochastic algorithm that digitizes the figure with a
vision-capable language model and uses the high-fidelity definition as constraints, returning an admissible set of individual patient datasets. A
solution is guaranteed to exist whenever the figure has been read correctly, so
a failure indicts either the published figure or its digitization. In simulation and in 79 arms from
published oncology trials, StochIKM reproduces the published risk tables and
medians more closely than LLM-improved IPDfromKM \citep{liu2021ipdfromkm}, a high-impact alternative.}

\maketitle

\section{Introduction}\label{sec:intro}
Individual patient data (IPD) are the currency of meta-analysis. For studies reporting time-to-event outcomes, access to IPD enables individual-participant-data meta-analysis \citep{stewart2002ipd,riley2010ipdma}; assessment of modelling assumptions \citep{royston2013rmst,uno2014beyond} and sensitivity analysis; extrapolation beyond the observed follow-up for the evaluation of health technologies \citep{latimer2013extrapolation};
hypothesis generation for future studies; and emulation of randomized comparisons between treatments never studied head to head \citep{hernan2016target,rahman2023external}.

The original IPD, however, are rarely accessible because of confidentiality and proprietary considerations. Even where data-sharing platforms exist, requests are often fulfilled for only a small number of applicants \citep{krumholz2016yoda,danchev2021sharing}. Reconstructing IPD from the published
Kaplan--Meier (K--M) curve offers a promising alternative, but the quality of such synthetic IPD is not guaranteed and requires careful evaluation. Where a ground truth is observed, as in supervised learning, accuracy is the distance from the truth. As the true IPD are rarely available, the question becomes how completely a reconstruction agrees with what the publication reports. Faith is ``the evidence of things not seen'' (Heb.~11:1, KJV). In information theory, \textit{fidelity} is defined as the expected distortion between a source and its reproduction under a stated distortion measure, and a code is judged against a fidelity criterion fixed in advance \citep{shannon1959fidelity,berger1971rate}. Although authors who reconstruct K--M figures have measured fidelity in different ways, the idea beneath is common to all: a reconstruction is faithful to the extent that it reproduces what the original publication reported.

Recovering information from published survival curves is not new.
\citet{parmar1998extracting} and
\citet{tierney2007practical} showed how a hazard ratio and its variance could be
recovered from a published curve together with the reported numbers at risk. \citet{guyot2012enhanced} introduced an algorithm that inverts the K--M product-limit formula one risk-table interval at a time, assuming that
censoring is spread uniformly within each interval, and so produces a
complete set of event and censoring times. \citet{liu2021ipdfromkm} refined that algorithm, correcting a boundary condition so that the K--M curve is mathematically plausible, and developed software that clinicians could use. Since then, reporting standards have improved the quality of survival figures \citep{pocock2002survival,morris2019kmunicate}, and vision-capable large language models (LLMs) are able to parse the elements of a figure simply by prompt engineering \citep{masry2022chartqa}. These developments reduce the burden of extracting information from published figures but do not resolve the statistical ambiguity inherent in reconstructing patient-level records from aggregate survival data.

Three methodological challenges remain. First, fidelity has been checked as an afterthought, against inconsistent criteria: agreement with the plotted
curve and the hazard ratio \citep{guyot2012enhanced}; with reported events, numbers
at risk, survival probabilities and medians \citep{liu2021ipdfromkm};
or with a synthetic ground truth \citep{zhao2025kmgpt}, rather than enforced as part of the reconstruction. Some reconstructed IPD therefore do not agree with the published risk table \citep{guyot2012enhanced,liu2021ipdfromkm}. Second, a published K--M curve does not identify the original IPD uniquely. Missing or overlapping censoring marks, sparse risk tables and the finite resolution of the figure admit many event and censoring configurations. Reported confidence intervals for the median carry additional information about sampling variability, but existing methods neither use them as constraints nor return more than a single IPD. Third, reconstruction depends heavily on digitization. Manual digitization is labor-intensive and prone to operator error \citep{liu2021ipdfromkm}, whereas automated extraction may degrade when figures are low-resolution rasters, curves overlap, or arms are distinguished only by line type or similar colours \citep{zhao2025kmgpt}.

To address these limitations, we propose StochIKM, an automated stochastic algorithm for generating high-fidelity IPD from published K--M curves. StochIKM makes three contributions to the three challenges. First, it defines fidelity by three constraints: each admissible IPD must reproduce the reported numbers at risk exactly, remain within a prespecified maximum absolute deviation from the extracted K--M curve, and agree within prespecified tolerances with the reported number of events, median survival and its confidence interval, when available. StochIKM enforces them within its loop. Second, StochIKM generates a set of $\Omega$ admissible datasets rather than one, so that the uncertainty of reconstruction can be propagated into downstream analysis. Third, under idealized conditions of internally consistent reporting, an error-free K--M figure and exact extraction of the plotted information, we establish the existence of at least one dataset satisfying the fidelity criteria.

\section{Methods}\label{sec:methods}
\subsection*{Step 1: Digitization}
StochIKM reconstructs IPD that satisfy the built-in high-fidelity constraints (``admissible'' IPD) from a published paper that contains a K--M curve and an accompanying risk table. First, a vision-capable LLM (in this paper, GPT-5) automatically extracts, or ``digitizes'', the information of interest. For simplicity, we present the notation for one treatment arm. The digitized elements fall into two categories. The first are numbers that can be checked directly against the original paper, denoted with an asterisk: i) from the risk table, the numbers of
patients at risk $r^*_0=n^*, r^*_1,\ldots,r^*_K$, where $n^*$ is the arm-specific sample size; ii) the median overall survival (mOS) $M^{*}$ and its interval $(M^{*}_L,M^{*}_U)$; and iii) the number of events $D^{*}$. These elements contain no digitization error. The second are elements that can be obtained only through the vision capability of the LLM, denoted with a hat: i) the coordinate system of the plotting area, that is, the pixel positions of the two axes and of their labelled ticks, by which a pixel is converted into a time and a survival probability; ii)
the curve of the arm, traced by its colour and read at times $0=\hat\tau_0<\hat\tau_1<\cdots<\hat\tau_J$ with heights $\hat s_j=\hat S(\hat\tau_j)$, where consecutive read times define the fine
intervals $F_j=(\hat\tau_{j-1},\hat\tau_j]$; and iii) the visible drops of that curve, each a vertical segment whose time and whose heights immediately above and
below it reveal that at least one event occurred there (Figure~\ref{fig:algorithm}B). These elements may contain digitization errors, and the resolution of the image bounds how large they can be: a drop can be located in time only to about one pixel column, $\delta_t$, and a height can be read only to about half the width of the drawn
line, $\delta_s$. Censoring marks may be printed but not used in this paper; we return to them in the Discussion.

\begin{figure*}[!tp]
\centering
\includegraphics[width=0.78\textwidth]{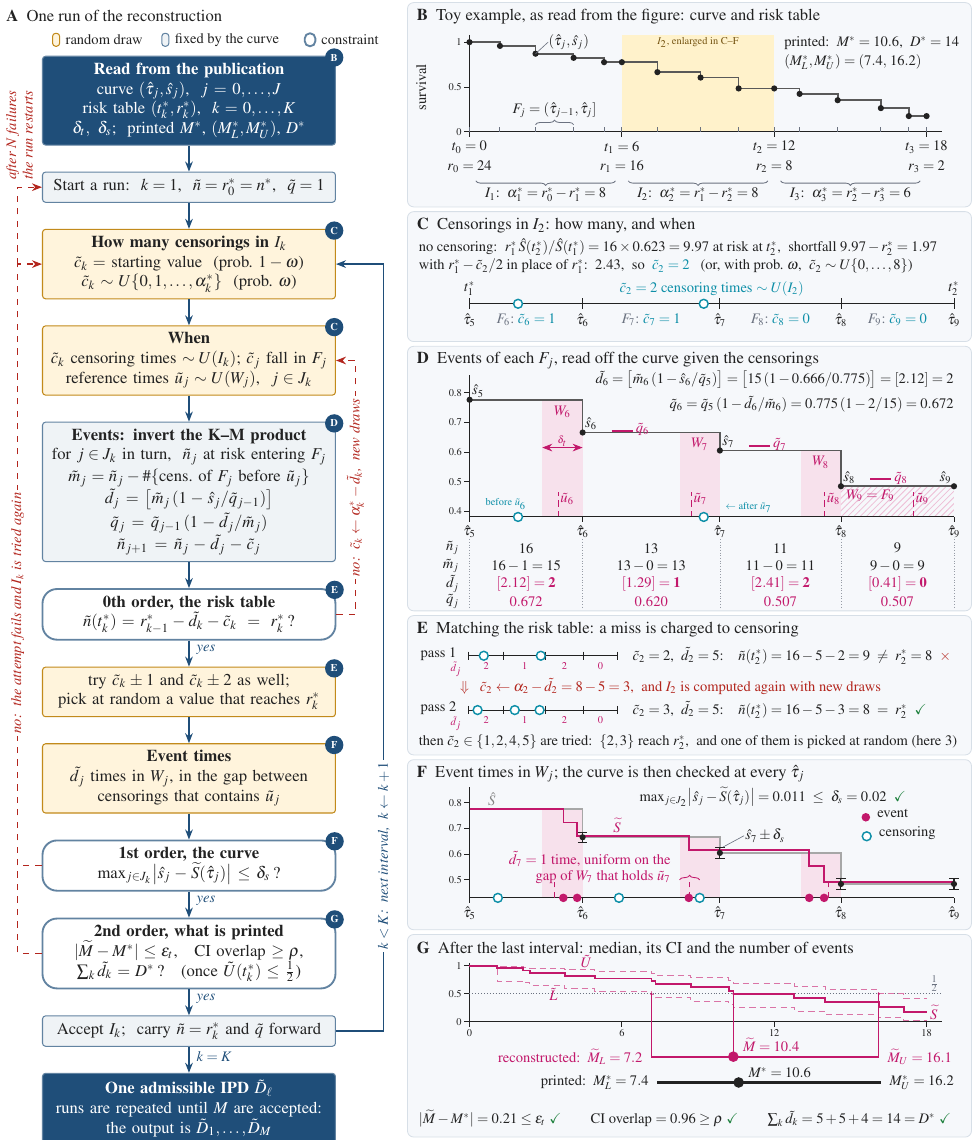}
\caption{The StochIKM algorithm. \textbf{A} Flow chart of one run. \textbf{B}--\textbf{G} A toy example carried through the steps of \textbf{A}: one arm of
24 patients with a risk table at 0, 6, 12 and 18 months, a printed median with its
confidence interval, and 14 printed events; \textbf{C}--\textbf{F} follow the second risk-table
interval, $I_2$.\label{fig:algorithm}}
\end{figure*}

\subsection*{Step 2: Reconstruction of admissible IPD}
A reconstruction is of high fidelity if it closely agrees with
the published K--M curve and the information commonly reported with it, and
an IPD that meets this standard is called admissible. Fidelity is imposed through three sequential constraints. A tilde denotes a quantity of the reconstruction: the survival $\tilde S$, the mOS and its interval ($\tilde M$, $\tilde M_L$, $\tilde M_U$), and the number of events $\tilde d_k$ in each risk-table interval. The zeroth-order, or risk-table, constraint requires that at every
risk-table time the reconstruction leaves exactly $r^*_k$ patients at risk (Figure~\ref{fig:algorithm}E),
\begin{equation}
  \tilde n(t^*_k) = r^*_k, \qquad k = 1,\ldots,K .
\end{equation}
The first-order, or curve, constraint requires that at every read time the K--M
curve of the reconstruction agrees with the digitized curve to within
$\delta_s$ (Figure~\ref{fig:algorithm}F),
\begin{equation}
  \max_{j=1,\ldots,J}\bigl|\hat s_j - \tilde S(\hat\tau_j)\bigr| \le \delta_s ,
\end{equation}
and that every event is placed within $\delta_t$ of the drop that revealed it.
The second-order, or reported-summaries, constraint uses the quantities the
publication prints alongside the figure (Figure~\ref{fig:algorithm}G),
\begin{equation}\label{eq:overlap}
\begin{gathered}
  \bigl|\tilde M - M^{*}\bigr| \le \varepsilon_t , \\
  \frac{\bigl|(\tilde M_L,\tilde M_U)\cap(M^{*}_L,M^{*}_U)\bigr|}
       {\bigl|(\tilde M_L,\tilde M_U)\cup(M^{*}_L,M^{*}_U)\bigr|} \ge \rho , \\
  \sum_{k=1}^{K}\tilde d_k = D^{*} ,
\end{gathered}
\end{equation}
where $(\tilde M_L,\tilde M_U)$ is the Brookmeyer--Crowley interval
\citep{brookmeyer1982}, and $\varepsilon_t$ and $\rho$ are fixed in advance.
Each part of (3) is used only when the publication prints the corresponding quantity, and an endpoint that is not reached is set to $t^*_K$; a figure with no
printed median and no printed number of events is reconstructed under (1) and (2)
alone. The algorithm below generates one candidate dataset at a time and keeps it
only if (1)--(3) hold, repeating until $\Omega$ admissible datasets have been collected.

The reconstruction proceeds one risk-table interval $I_k=(t^*_{k-1},t^*_k]$ at a time, from $I_1$ to
$I_K$, carrying forward the number of patients still at risk and the
reconstructed survival (Figure~\ref{fig:algorithm}A). Within $I_k$ it
must determine the number of events $\tilde d_k$ and the number of censored patients
$\tilde c_k$, whose sum must be the attrition $\alpha^*_k = r^*_{k-1}-r^*_{k}$ obtained in Step~1.

\begin{description}\setlength{\itemsep}{1pt}

\item[Initialization.] Set $\tilde q_0 = 1$ and $\tilde n_1 = n^{*}$, and recall
that each $I_k$ is the union of the fine intervals $F_j$, $j \in J_k$
(Figure~\ref{fig:algorithm}B).

\item[Step 2.1: the censoring.] With no censoring,
$r^{*}_{k-1}\hat S(t^{*}_k)/\hat S(t^{*}_{k-1})$ patients would be at risk at
$t^{*}_k$, and the shortfall against $r^{*}_k$ gives a starting value for
$\tilde c_k$, corrected for the fact that the events see about
$r^{*}_{k-1}-\tilde c_k/2$ patients rather than $r^{*}_{k-1}$. With probability
$\omega$ this value is replaced by a uniform draw from
$\{0,1,\ldots,\alpha^*_k\}$; this randomization is what lets the search reach every possible configuration, as Proposition~\ref{prop:exists} requires. The censoring times are drawn
independently and uniformly on $I_k$, and $\tilde c_j$ of them fall in $F_j$
(Figure~\ref{fig:algorithm}C).

\item[Step 2.2: the events.] Events of $F_j$ are placed in the event window $W_j$, equal to
the last $\delta_t$ of $F_j$ when the curve drops at $\hat\tau_j$ and to all of
$F_j$ otherwise,
and a reference time $\tilde u_j$, uniform on $W_j$, fixes which censorings of
$F_j$ come before them (Figure~\ref{fig:algorithm}D). Inverting the K--M product
one fine interval at a time gives
$\tilde m_j=\tilde n_j-\#\{\text{censorings of } F_j \text{ before } \tilde u_j\}$ and
\[
\begin{gathered}
  \tilde d_j=\Bigl[\tilde m_j\Bigl(1-\frac{\hat s_j}{\tilde q_{j-1}}\Bigr)\Bigr],
  \\
  \tilde q_j=\tilde q_{j-1}\Bigl(1-\frac{\tilde d_j}{\tilde m_j}\Bigr),
\end{gathered}
\]
with $\tilde n_{j+1}=\tilde n_j-\tilde d_j-\tilde c_j$. Given the censorings,
this step is deterministic.

\item[Step 2.3: matching the risk table.] If
$\tilde n(t^{*}_k)=r^{*}_{k-1}-\tilde d_k-\tilde c_k$ is not $r^{*}_k$, the
discrepancy is charged to censoring, $\tilde c_k \leftarrow \alpha^*_k-\tilde d_k$,
and $I_k$ is recomputed (Figure~\ref{fig:algorithm}E). Once $r^{*}_k$ is
attained, $\tilde c_k \pm 1$ and $\tilde c_k \pm 2$ are tried as well, and one of
those that also reach $r^{*}_k$ is chosen at random.

\item[Step 2.4: fidelity check.] Constraint (1) now holds by construction;
constraint (2) is checked over $j \in J_k$ (Figure~\ref{fig:algorithm}F); and
constraint (3) is checked as soon as it is decidable, that is, once the upper
confidence limit of the reconstructed curve has fallen to one half
(Figure~\ref{fig:algorithm}G). A failed interval is attempted again from Step~2.1,
and after $N$ failures the run is discarded and restarted from $I_1$ (dashed
paths of Figure~\ref{fig:algorithm}A).

\end{description}

\noindent Steps 2.1--2.4 are repeated for $k=1,\ldots,K$; a run that reaches the end of $I_K$
yields one admissible IPD $\tilde D_\ell=\{(\tilde y_i^{(\ell)},\tilde \Delta_i^{(\ell)})\}_{i=1}^{n^*}$, and runs continue until $\tilde D_1,\ldots,\tilde D_\Omega$ have been
collected.

\begin{proposition}[Under perfect digitization, StochIKM guarantees a solution.]\label{prop:exists}
Suppose the risk table and the printed summaries are transcribed correctly, and every
survival probability $\hat s_j$ read from the figure is within $\delta_s$ of the curve
printed in the publication (perfect digitization). Then the original patient data of the trial satisfy constraints (1)--(3); in particular, the admissible set is not empty.
\end{proposition}

\begin{proof}
Write $x^{0}$ for the original data. Since the risk table is read correctly,
$\tilde n(t^{*}_k) = r^{*}_k$ holds at $x^{0}$ for every $k$, which is (1), and
exactly $\alpha^{*}_k$ patients leave the risk set in each $I_k$. The K--M curve
of $x^{0}$ is the published curve $S$, so at every read time
$|\hat s_j - S(\hat\tau_j)| \le \delta_s$ by assumption, which is (2); each event
of $x^{0}$ produced the drop that revealed it, so it lies in that drop's window.
The printed summaries $M^{*}$, $(M^{*}_L,M^{*}_U)$ and $D^{*}$ are computed from
$x^{0}$, so (3) holds at $x^{0}$ up to the rounding of the printed digits, which
$\varepsilon_t$ absorbs.
\end{proof}

\subsection*{Step 3: Relaxation}
Every quantity denoted with a hat is read from an image, which may be of poor quality or even contain mistakes, and the LLM may misread a legend or a colour. The premise of Proposition~\ref{prop:exists}, the perfect digitization assumption, fails and the admissible set may be empty. Repeated failure therefore has two possible causes: the search has not yet found a solution, or the digitization is
inconsistent with the published summaries and no solution exists. Step~3, relaxing constraints (1)--(3) lower the thresholds to a certain extent, and still produce result of fidelity given imperfect digitization. If $R$ runs produce no admissible dataset, the constraint
responsible for most of the failures is identified and relaxed by one level, and
the search restarts; the levels are fixed in advance and applied in the same
order for every figure. The second-order constraint has three parts and they are relaxed separately. The number of events  is the commonest cause of failure, a dataset is
accepted if $\bigl|\sum_k \tilde d_k - D^{*}\bigr| \le \tau D^{*}$, with $\tau = 0.035$, and if
that is still not enough, $\delta_s$ is widened as for the first-order constraint while the
tolerance is kept. The number of events is a sum over
every interval, and configurations that meet every other constraint often miss it by a few
events, so a small tolerance raises the efficiency of the search considerably while the risk
table, the curve, the median and its interval keep their full fidelity. A discrepancy in number of events per arm may be attributed to a systematic mis-reading besides the digitization: the reported total number of events may count events that occurred after the last time point the figure shows, a practice that varies with the authors and the journals. The tolerance $\tau$ is
the largest discrepancy observed in the case study, 9 of 262 events. For the median and its
interval, $\varepsilon_t$ is doubled and $\rho$ lowered to $0.6$, then $\varepsilon_t$
doubled again and $\rho$ lowered to $0.4$, then the interval condition is dropped, then the
median condition is dropped. For the first-order constraint, $\delta_s$ is multiplied by
$1.5$ at each level, up to a ceiling $\delta_s^{\max}$. The zeroth-order
constraint is never relaxed. If no admissible dataset is obtained once every level has been exhausted, the figure is declared irreconcilable and no IPD is returned. A reconstruction obtained after relaxation is reported together with the
levels used.

Throughout, the tolerances $\delta_t$ and $\delta_s$ are fixed by the resolution
of the figure, one pixel column and half the width of the drawn line
respectively, and $\varepsilon_t$ is $\delta_t$ plus half the last printed digit
of $M^{*}$. The remaining constants are used for every
figure in this paper: the probability of the uniform draw in Step~2.1 is
$\omega=0.1$; an interval is attempted $N=200$ times before the run is discarded;
the required overlap of the two intervals is $\rho=0.8$; the portfolio size is
$\Omega=50$; relaxation begins after $R=40\Omega$ runs without an admissible dataset and
stops at the floors $\rho_{\min}=0.4$ and $\delta_s^{\max}=0.10$.

\section{Case Study}\label{sec:case}
We applied StochIKM to 108 arms in 53 overall-survival figures of 28 trials, 18 in
advanced gastric or gastro-oesophageal junction cancer and ten in advanced breast cancer,
published between 2017 and 2025 in eleven journals; most are phase III trials. A figure shows the whole trial
population or a biomarker subgroup, and an arm has between 14 and 792 patients. All trials
are listed in the References with their ClinicalTrials.gov identifiers. Details description can be found in \textbf{Supplementary Materials}. StochIKM is compared with LLM-improved IPDfromKM, that is, IPDfromKM \citep{liu2021ipdfromkm} given the same traced curve and the same risk table as input, which removes the manual curve-tracing error of the original software.

\begin{figure*}[!tp]
\centering
\includegraphics[width=0.78\textwidth]{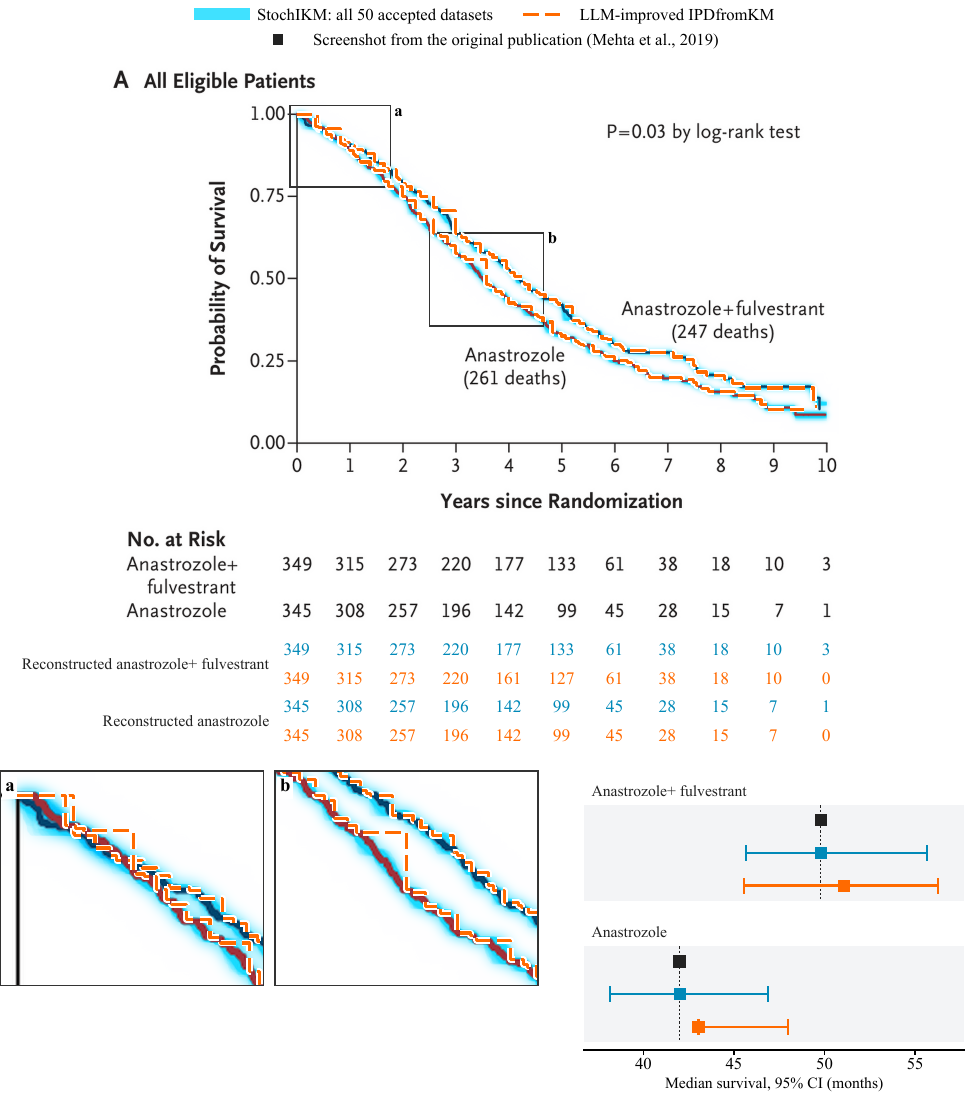}
\caption{Reconstructions drawn over a published overall-survival figure of
NCT00075764 \citep{NCT00075764}: the K--M curves of the 50 admissible IPD of StochIKM (light blue, forming a band) and of the single IPD of LLM-improved IPDfromKM (orange, dashed), their numbers at risk under the published risk table, two enlarged regions (a, b), and the median survival with its 95\% confidence interval (black, published). Time is in years; medians are in months.\label{fig:comparison}}
\end{figure*}

Figure~\ref{fig:comparison} shows one of the figures, taken from
NCT00075764 \citep{NCT00075764}. In the enlargements, the 50 admissible curves of
StochIKM form a narrow band that follows the published curve step by step,
whereas the single curve of LLM-improved IPDfromKM moves in fewer and larger
steps. StochIKM reproduces the risk table of both arms in all 50 IPD, and in the
anastrozole plus fulvestrant arm every IPD also carries the 247 deaths printed on the
figure; LLM-improved IPDfromKM reproduces neither risk table. The published medians are 49.8 and 42.0 months.
StochIKM gives 49.8 and 42.0, and LLM-improved IPDfromKM 51.1 and 43.0; across
its 50 IPD the median of StochIKM varies by 0.39 and 0.40 months, which is the
uncertainty the figure itself leaves. This publication prints no confidence
interval for either arm, so the intervals in the lower right panel are those of
the reconstructions alone. LLM-improved IPDfromKM places all events of a fine interval at the drop that revealed them, so its
reconstruction has many ties: 65 distinct event times in the anastrozole
arm, one of them carrying 24 events. A step that large can take the curve across one half in one jump; the Brookmeyer--Crowley lower limit then reaches one half at the same
event time, and the interval collapses onto
the median, drawn with no
left whisker. StochIKM places each event inside its own window $W_j$, which
yields 264 distinct event times in the same arm and an interval with
both ends. Other cases can be found in \textbf{Supplementary materials}.

The curves of 104 of the 108 arms could be traced. The four that could not are the arms of
one figure (NCT02545504), in which two pairs of arms share a colour and differ only in line
style, so that no tracing by colour can tell them apart (discussed in detail in \textbf{Supplementary materials}). StochIKM returned at least one
admissible IPD for every one of the 104 traced arms, and the full admissible set of $\Omega=50$ for 97 of
them; in the other seven the run limit was reached with between 1 and 49 admissible IPDs because in those cases, constraints (1)--(3) are tighter. Among those
104 arms, 97 print a median survival and LLM-improved IPDfromKM reached a median in all of
them, so the two methods are compared over these 97 arms, of which 71 also print both
bounds of the confidence interval (Table~\ref{tab:comparison}). StochIKM reproduces the risk
table in every arm and LLM-improved IPDfromKM in only 44. The median error in the median
survival is 0.00 (not perfect, 0.00 as a result of rounding to 2 demcimals) month against 0.10, and the largest error 0.91 against 1.32 months. Over the
71 arms that report an interval, the reconstructed intervals agree with the published ones
equally well (median overlap 0.94 for both).

\begin{table}[!t]%
\centering
\caption{Reconstruction of 97 published arms by StochIKM and LLM-improved IPDfromKM.\label{tab:comparison}}%
\begin{tabular*}{\columnwidth}{@{\extracolsep\fill}l D{.}{.}{1.2} D{.}{.}{1.2}@{}}
\toprule
 & \multicolumn{1}{c}{\textbf{StochIKM}$^{a}$}
 & \multicolumn{1}{c}{\begin{tabular}[b]{@{}c@{}}\textbf{LLM-improved}\\ \textbf{IPDfromKM}$^{b}$\end{tabular}}\\
\midrule
Risk table reproduced exactly
   & \multicolumn{1}{c}{97/97} & \multicolumn{1}{c}{44/97}\\
\multicolumn{3}{@{}l}{Absolute error in the median survival (months)}\\
\hspace{1em}Median            & 0.00 & 0.10\\
\hspace{1em}Mean              & 0.05 & 0.21\\
\hspace{1em}90th percentile   & 0.16 & 0.56\\
\hspace{1em}Maximum           & 0.91 & 1.32\\
\hspace{1em}Within 0.5 month
   & \multicolumn{1}{c}{96/97} & \multicolumn{1}{c}{84/97}\\
\hspace{1em}Within 1 month
   & \multicolumn{1}{c}{97/97} & \multicolumn{1}{c}{93/97}\\
\multicolumn{3}{@{}l}{Agreement of the confidence intervals, 71 arms$^{c}$}\\
\hspace{1em}Median overlap          & 0.94 & 0.94\\
\bottomrule
\end{tabular*}
\begin{tablenotes}
\item[$^{a}$] The accepted dataset whose median is closest to the published
value (50 datasets per arm; fewer in 7 arms); 55 arms met the full criterion
and 42 needed relaxation.
\item[$^{b}$] Given the same LLM-traced curve and risk table as StochIKM.
\item[$^{c}$] Intersection over union of the reconstructed and
published intervals.
\end{tablenotes}
\end{table}

Of the 101 arms whose curve could be traced, StochIKM returned 50 admissible IPD for 86 arms; the remaining 15 arms returned an empty admissible set, because of the quality of the published figure, the tracing of the arm by the LLM, or both (Proposition~\ref{prop:exists}). 81 of the 86 arms print a median survival, and in two of these LLM-improved IPDfromKM reached no median, so the two methods are compared over the remaining 79 arms, of which 61 also print both bounds of the confidence interval (Table~\ref{tab:comparison}). Over these 79 arms StochIKM reproduces the risk table in every arm and LLM-improved IPDfromKM in only 33. The median error in the median survival is 0.01 month against 0.13, and the largest error 0.91 against 5.26 months. Over the 61 arms that report an interval, the reconstructed intervals agree with the published ones about equally well (median overlap 0.93 and 0.92).

\section{Simulation}\label{sec:simulation}
The simulation asks one question: when the figure is read without error, does a method
give back the data behind it? Each replicate is one arm. Event times are Weibull with
shape $1.2$ and a median of $12$ months; patients drop out at an exponential rate chosen
to reach a target censoring fraction, and everyone still in the study is censored at $36$
months. From these data we take what a publication prints: the K--M curve, read
exactly at every step; the risk table at $K+1$ equally spaced times; the median and its
$95\%$ confidence interval, rounded to one decimal; and the number of events. StochIKM
and IPDfromKM \citep{liu2021ipdfromkm} are given the same input. The baseline has $350$ patients,
$30\%$ censoring and $K=8$, and each factor is varied in turn: $50$ or $1000$ patients,
$10\%$ or $50\%$ censoring, $K=4$ or $12$. Every scenario has $1000$ replicates.
Over time, we take the difference between the
reconstructed K--M curve and that of the data, and report its mean over the
replicates (bias) and its standard deviation (empirical SE). As single numbers, we report
how often the risk table and the number of events are reproduced exactly, the absolute
error in the median, and the overlap of the reconstructed and the original confidence
interval of the median, defined as in Equation~\eqref{eq:overlap}.

\begin{figure*}[!t]
\centering
\includegraphics[width=0.92\textwidth]{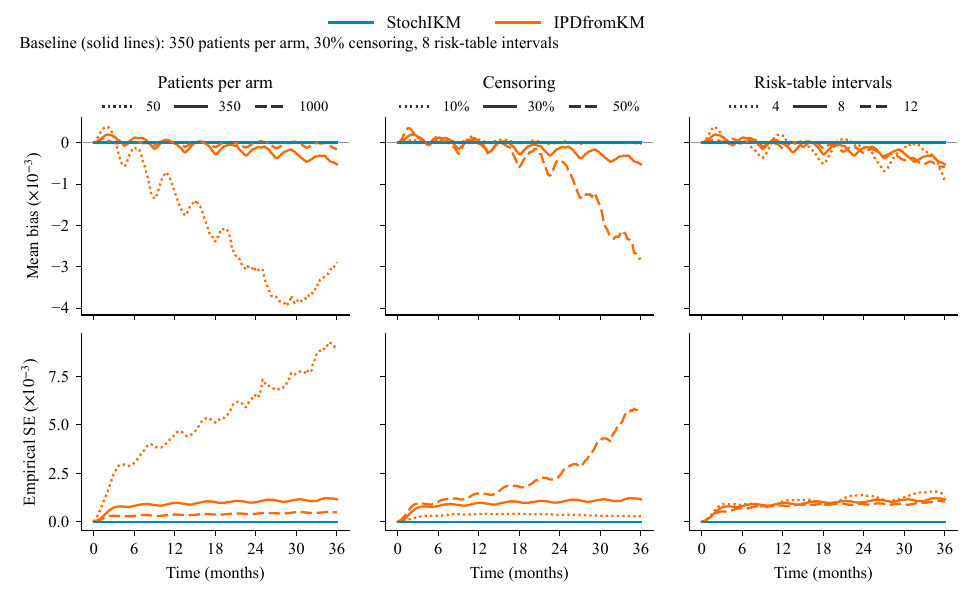}
\caption{Simulation with perfect digitization. Top row: mean over 1000
replicates of the difference between the reconstructed K--M curve and that of the simulated data (bias). Bottom row: its standard
deviation (empirical SE). StochIKM (blue) is at zero throughout and
its three lines coincide.\label{fig:simulation}}
\end{figure*}

StochIKM returned the data exactly in every replicate of every scenario: its bias and its
empirical SE are zero at all times, and it reproduces the
risk table, the number of events, the median and its confidence interval without error
(Figure~\ref{fig:simulation}). IPDfromKM is close but not exact. Its curve lies below that
of the data, and the bias and the empirical SE grow with time; both are largest with few
patients ($-0.004$ and $0.009$ at $n=50$) and with heavy censoring ($-0.003$ and $0.006$
at $50\%$), and the number of risk-table intervals changes them little. It reproduces the
risk table in every replicate, but the number of events, although it is given, in only
$4\%$ to $97\%$ of them, least often with many patients, heavy censoring or few risk-table
intervals. Its error in the median is $0.01$ to $0.10$ month and its confidence intervals
overlap the original ones by $0.98$ to $0.99$.

\section{Discussion}\label{sec:discussion}
Reconstruction methods routinely claim high fidelity, but each has measured it
differently, and no criterion has been agreed
upon \citep{guyot2012enhanced,liu2021ipdfromkm,zhao2025kmgpt}. Our principal
contribution is to propose one: a reconstruction is admissible when it reproduces
the risk table exactly, follows the digitized curve to within the resolution of
the image, and respects the summaries printed beside the figure. Because that criterion does not identify a single
dataset, StochIKM returns a portfolio of $\Omega$ admissible ones, and the spread
across them measures what the publication left undetermined.

We build directly on \citet{liu2021ipdfromkm}, whose algorithm depends on manual,
software-assisted tracing and returns a single IPD, and whose authors
anticipated both the automation of the reading and the multiplicity of the
answer. Vision-capable language models have made the first possible, and the
constrained sampling scheme developed here makes the second explicit. Over 97
published arms the reconstruction reproduced every risk table and matched the
published median to within one month in all of them, and
Proposition~\ref{prop:exists} shows that, when the figure is traced exactly, an
admissible set exists. A failure of producing result therefore indicates either an imperfect figure-tracing or imperfect quality of the source figure.

Limitations lie in the figure-reading part: we are not able to use censoring marks. Many figures draw them, and they carry exactly the information determining IPDs. The K--M curve and the risk table determine how many patients
are censored between two events but not when. Authors and journals draw these marks in
different ways, and adjacent marks merge into one, so the vision-language models available to
us cannot yet read them reliably. If a highly capable large language model could read the marks, then it would boost the algorithm on efficiency: with the censoring times given, the events in each interval would follow
deterministically from the steps of the curve, so the search for $\tilde c_k$ would no longer be needed. Under perfect digitization
it would recover the original true patient data themselves, up to tied events
(\textbf{Supplementary material)}; a partial reading, which separates some marks and only
counts the others, would already narrow the admissible set. More broadly, what limits
StochIKM is the vision of the model that reads the figure. Given how quickly vision-language models have improved, we expect this
limitation to ease, and the algorithm is ready to take advantage of better readings as they
become available.

The algorithm is nonetheless ready for use at scale. A figure whose reading cannot be reconciled with its own
printed summaries becomes evidence about the publication, so reconstruction at
scale can flag figures that no patient data could have produced, the
concern behind systematic screening of trials for fabricated data
\citep{carlisle2017}. Reconstructed records restore the time-to-event analyses
that aggregate reporting forecloses, including restricted mean survival and the
emulation of comparisons never made directly
\citep{royston2013rmst,hernan2016target}. And a portfolio of admissible datasets
is a natural input to Bayesian borrowing from historical controls
\citep{ibrahim2000power,viele2014historical}, where the question becomes how much of a conclusion survives
across every dataset the figure permits.

\section{Code Availability}
The source code is available from emailing robert.ke@yale.edu upon reasonable request. A web-based user interface is under development and will be made publicly available. Researchers interested in early access to a pre-release version of the interface are welcome to contact the corresponding author.

\nocite{NCT00075764,%
        NCT01234337,%
        NCT01630083,%
        NCT01774786,%
        NCT01928394,%
        NCT02107703,%
        NCT02494583,%
        NCT02545504,%
        NCT02746796,%
        NCT02767661,%
        NCT02872116,%
        NCT02954536,%
        NCT03125902,%
        NCT03504397,%
        NCT03523585,%
        NCT03529110,%
        NCT03615326,%
        NCT03653507,%
        NCT03675737,%
        NCT03694522,%
        NCT03745170,%
        NCT03777657,%
        NCT03802591,%
        NCT03901339,%
        NCT04191499,%
        NCT04639986,%
        NCT04662710,%
        NCT05008783}

\bibliography{refs}

\end{document}


\maketitle
\section{Prompt engineering in the digitization process}\label{sec:S1}
\paragraph{Division of labour.} A published panel carries two kinds of
information, and they are read by different instruments. Text---axis labels and
ticks, arm names and colours, the numbers-at-risk table, and any median printed
in the figure---is transcribed by a vision-capable language model, here GPT-5, which copies
symbols accurately but localizes a drawn line poorly. The curve itself is read by
pixel tracing, which localizes to within a pixel but cannot name what it sees.
The language model therefore supplies the coordinate system and the metadata that
the tracer needs, together with a coarse reading of each curve that tells the tracer
which line to follow (Section~\ref{sec:S2}), and the tracer supplies the survival values.

\paragraph{Stage 1: metadata under a schema.} The first call returns a single
JSON object constrained by a fixed schema: axis labels, units and the largest
labelled tick; the pixel coordinates of the plotting frame; one entry per arm
giving its printed name, colour name and hexadecimal colour, line style and the
time at which its curve ends; the time headers of the risk table and each arm's
row, truncated at the first blank cell rather than padded; every visible vertical
step of each curve, with the survival values immediately above and below it; and
the printed median with its interval, level and unit when the figure shows one.
Two instructions do the work of many: transcribe what is printed and do not guess
values that are not visible, and end a risk-table row at the first missing cell
instead of inventing entries. Schema constraints make the reply parseable and
make omissions explicit.

\paragraph{Stage 2: curve values in free form.} Reading a value off a plot is a
measurement, not a transcription, and constraining it to a schema degrades it: in
our earlier system, schema-constrained readings were $9$--$20$ percentage points
from the truth, against $1.2$--$3.5$ for the same model asked to reason in prose
before answering. Stage 2 therefore asks the model to work through one arm and at
most twenty times per call, naming the labelled gridline immediately above the
curve and the one immediately below it before giving each value, and to close
with a single line in a fixed format from which the numbers are parsed. Four
rules are stated in the prompt: follow the drawn line and read its centre where
censoring marks thicken it; never report a value larger than the previous one,
since a Kaplan--Meier curve does not rise; read the level above a vertical step
just before it and the level below just after; and report what is drawn, without
adjusting anything to match a median quoted in the paper. The last rule protects
the second-order check, which is informative only if the curve is read
independently of the median it is later compared with.

\paragraph{Where the curve is read.} The language model reads each arm at every
risk-table time, at $24$ equally spaced times across the observed range, and at a
pair of times $\pm 0.004\,t_{\max}$ around each step reported in stage 1. These
readings are not used as the curve: they guide and check the pixel trace
(Section~\ref{sec:S2}), whose own readings, at the risk-table times and the same $24$ equally
spaced times, together with the steps it detects, are what the reconstruction
uses. Small figures are upscaled before they are sent, so that the plotting frame
is at least $600$ pixels high, since a plot region of a few hundred pixels cannot
support a reading finer than its own resolution.

\paragraph{The prompts}
\begin{quote}\small\ttfamily\raggedright
\singlespacing
\textbf{Stage 1, system message.} "You digitise Kaplan-Meier figures. Reply with
a single JSON object matching this schema exactly: [JSON schema; required fields
are axes, arms and risk\_time\_points]"

\medskip
\textbf{Stage 1, user message.} "You are digitising a Kaplan-Meier survival
figure from a clinical trial publication. Transcribe what is printed. Do not
guess values that are not visible. [Blocks 1--2 request the axis labels, units,
ticks and the pixel coordinates of the plotting box, and one entry per arm
giving its printed name, colour, line style and the time its curve ends.]
3. NUMBER AT RISK TABLE [\ldots] If a cell is blank or missing, END the list
there. Do not invent values and do not pad with zeros.
4. VISIBLE DROPS - for each arm, drops: every step of the curve that you can see
as its own vertical segment. For each give t (time of the vertical segment,
x-axis units), s\_top (survival just above the segment) and s\_bottom (just
below) [\ldots] Where the curve descends in steps too small to see one by one,
do not list them. Censoring tick marks are not drops.
5. PRINTED MEDIANS - only if the figure itself prints a median survival for an
arm (in the legend, an inset table or an annotation): printed\_median with
median, ci\_lower, ci\_upper (numbers, or "NR" if not reached), ci\_level and
unit (the time unit printed next to the median [\ldots] it may differ from the
x-axis unit). Otherwise omit it. [Block 6 asks for the endpoint name, a
self-confidence score and free-text notes.]"

\medskip
\textbf{Stage 2, system message.} "You read values off scientific plots
carefully and show your reasoning."

\medskip
\textbf{Stage 2, user message.} "[The axis descriptions are filled in from stage
1.] Read ONLY the \{arm\}. Ignore every other curve, any confidence band, and
the number-at-risk table below the plot. Work through the times one at a time,
out loud. For each time: 1. Find that time on the x axis and move straight up to
where THIS curve is drawn. 2. Name the labelled gridline immediately ABOVE the
curve there and the one immediately BELOW it. 3. Say roughly where between those
two the curve sits, then give the value. Rules that matter: - Follow the drawn
line. Where censoring tick marks thicken it, read its centre. - A Kaplan-Meier
curve never rises. If a value you are about to write is higher than the previous
one, re-read it. - Where a time lies just before a vertical step, read the level
above the step; just after it, the level below. - Report on the 0-100 percent
scale even if the axis is drawn 0.0 to 1.0. - Report what is drawn. Do not
adjust anything to match a median quoted in the paper. Times to report:
\{times\} Finish with one line in exactly this format, nothing after it: FINAL:
\{t\}=<v> [\ldots] Each <v> is a number from 0 to 100 with at most one decimal
place."
\end{quote}

\section{Pixel tracing guided by the language-model readings}\label{sec:S2}
The tracer turns the screenshot into the curve values $\hat s_j$ at the read times
$\hat\tau_j$ and the drops of Section~2 of the main text, in four steps.

\paragraph{Step 1: calibration.} The plotting frame is located from the axis lines: the lowest dark
horizontal line spanning at least $45\%$ of the image width, and the leftmost dark
vertical line spanning at least $35\%$ of its height. Tick marks are searched for in a band
of $2$ to $12$ pixels on each side of each axis line, separately, so that a thick or
blurred axis line does not hide them; a tick is a run of at least three dark pixels,
narrower than $0.8\%$ of the axis or five pixels, whichever is larger. The map from a
time to a pixel column is the straight line through two detected ticks carrying two
of the labels read in stage 1, chosen among all such pairs to place the most labels on
detected ticks (to within three pixels), with the full labelled range spanning between
$0.7$ and $1.3$ times the axis length. When ticks are drawn at every minor step and
labelled only at some of them, several scales place every label on a tick; the one
whose first and last labels fall at the ends of the axis line is taken. The survival
axis is calibrated in the same way. A time axis that cannot be calibrated from its
ticks makes the figure unusable, since the ends of the drawn axis line do not in
general coincide with the first and last labels; a survival axis may fall back to the
line ends, which happened in three of the $53$ figures.

\paragraph{Step 2: colours.} The non-white pixels inside the frame are clustered
into at most ten candidate colours, and every pixel is assigned to its nearest
candidate. For each candidate and each arm, the pixels of that colour are split into
connected pieces, and the curve of the arm in that colour is assembled from them.
The coarse readings of stage 2 decide which pieces belong to it: a piece whose median
distance from the readings exceeds $15$ percentage points is not part of the curve,
however long, flat and well placed it is. This single rule removes the objects that
otherwise mimic a curve: the short line segments of a legend, the axis lines and
gridlines, and text or median annotations printed in the colour of an arm. The widest
piece that follows the readings is taken first, and the remaining pieces are added
from the widest down, each only if it neither makes the curve rise nor lies below
its later course. Each arm is then given one colour by minimizing, over all
one-to-one assignments, the sum of three costs: the distance from the colour named in
stage 1, eight units per percentage point of distance between the assembled curve and
the readings, and $400$ units times the fraction of the read time range the curve
leaves uncovered. The last term separates a curve from the lighter, anti-aliased fringe
of the same curve, whose pixels follow the readings closely but only in fragments.
When the language model returns fewer than four readings for an arm, its curve is
assembled from the widest piece that starts near time zero, without the readings;
this happened for four arms. When two arms are printed in the same colour and told
apart only by line style, no colour can separate them, and both are declared
untraceable.

\paragraph{Step 3: the curve.} In each pixel column the curve is the median row of
its pixels, converted to a survival probability through the calibration, smoothed
by a three-column median, and made non-increasing by isotonic regression. Columns
in which the curve is hidden behind another arm are interpolated from the curve's
own neighbours; a run of more than two hidden columns is marked, and no reading or
drop is taken from it. A curve hidden behind the other arm from time zero, which
has no pixel of its own near the origin, is anchored at survival one and marked
hidden up to its first own pixel; this happened for $22$ arms.

\paragraph{Step 4: drops and readings.} A drop is a run of columns over which the
fitted curve falls by at least $1.5$ pixels (and at least $0.003$ in survival) and
which is not hidden; its time is the column at which it ends, and its top and
bottom are the survival values on either side. The curve is read at every
risk-table time and at $24$ equally spaced times, skipping hidden columns. The
resolution of the trace sets the tolerances of Section~2 of the main text: $\delta_t$ is one pixel
column converted to time.

\paragraph{When a trace is used.} A trace is used only if all of the following hold.
The time axis was calibrated from ticks. The curve starts at survival at least $0.9$
and its raw value at the origin does not exceed $1.05$; for an anchored curve, its
first own pixels lie above $0.6$, which separates a curve hidden at the start from a
flat tail mistaken for one. It covers at least $40\%$ of the columns between its
first and last pixel and contains at least one drop. Its median distance from the
readings of stage 2 is at most ten percentage points. Over the $100$ arms traced
with readings this distance had a median of $1.8$ points, a 90th percentile of
$4.7$ and a maximum of $8.0$. Times at which the trace and the readings differ by
more than five percentage points are recorded in the output for inspection. A
figure that fails these checks is reported as unreadable rather than reconstructed
from a reading known to be unreliable.

\section{The case-study output}\label{sec:S3}
Every figure of the case study has its own folder in \texttt{output/}, named
\texttt{NCT\textit{xxxxxxxx}\_Fig\textit{n}\_Panel\textit{X}} after the trial
registration number, the figure and the panel; the screenshot and its optional
metadata file are in \texttt{inputs/} under the same name. A folder contains:
\begin{description}\setlength{\itemsep}{0pt}
\item[\texttt{extraction.json}] The two language-model stages: axes and their ticks,
the plotting frame, and for each arm its name, colour, line style, risk-table row,
the drops seen in stage 1, the stage-2 readings and any printed median.
\item[\texttt{trace.json}] The pixel trace: the calibration of both axes and how it
was obtained, the candidate colours, and for each arm its colour, its status
(used or not, with the reason), its distance from the readings, its coverage and
hidden columns, its readings and its drops.
\item[\texttt{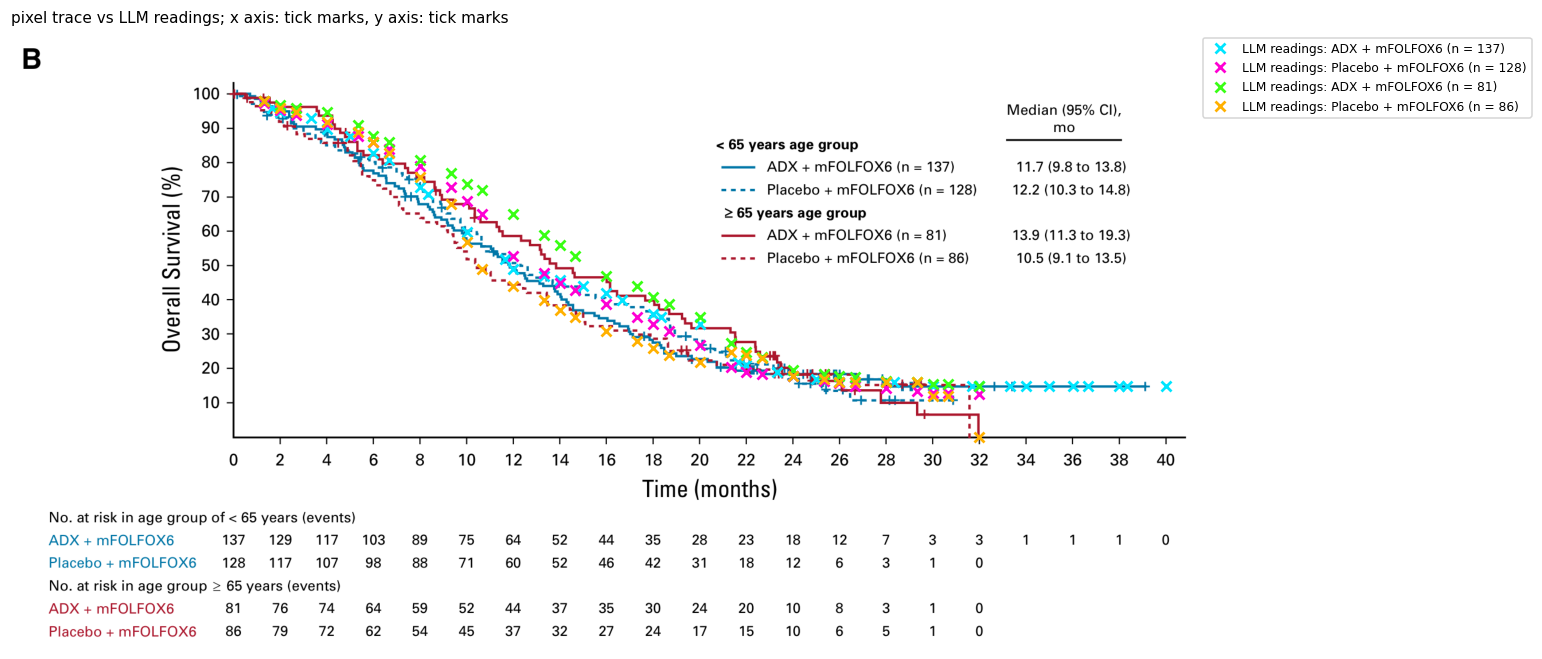}, \texttt{\textit{stem}\_trace.pdf}] The pixel trace drawn
over the screenshot, with the stage-2 readings as crosses: the first check a reader
should make.
\item[\texttt{arm\textit{i}\_\textit{name}\_ipd.csv}] The admissible IPD of arm
$i$, one row per patient, with columns \texttt{dataset} ($1$ to $50$),
\texttt{patient}, \texttt{time} (in the unit of the time axis) and \texttt{event}
($1$ death, $0$ censored).
\item[\texttt{summary.json}] Everything about the reconstruction: the parameters, the
risk table and any corrections made to it, the reported median, interval and number
of events, the number of runs and accepted datasets, the failures by cause, the
relaxation levels used, and the distribution of the reconstructed median and its
interval over the accepted datasets.
\item[\texttt{overlay.png}, \texttt{\textit{stem}\_overlap.pdf}] All accepted
reconstructions drawn over the screenshot, with the risk table and the summary.
\item[\texttt{compare/}] The comparison with LLM-improved IPDfromKM: its input
(\texttt{arm\textit{i}\_points.csv}, the traced points; \texttt{arm\textit{i}\_risk.csv},
the risk table), its output (\texttt{arm\textit{i}\_ipdfromkm.csv}), the R script
that produced it, the metrics of both methods side by side (\texttt{metrics.txt}),
and the two reconstructions drawn over the screenshot (\texttt{both.png},
\texttt{stochikm.png}, \texttt{ipdfromkm.png}).
\end{description}
The published, StochIKM and LLM-improved IPDfromKM median and interval of every
reconstructed arm, from which Table~1 of the main text is computed, are collected in
\texttt{output/comparison\_by\_arm.csv}. The spreadsheet
\texttt{StochKM\_relaxation\_summary.xlsx} lists every one of the $108$ arms with
its category (Section~\ref{sec:S4}), the failure under the default constraints, and the relaxation used.

\section{What happened to each of the 108 arms}\label{sec:S4}
Each arm falls into one of five categories (Table~\ref{tab:categories}).
\begin{description}\setlength{\itemsep}{0pt}
\item[A, not traced.] The curve could not be traced, and no reconstruction was
attempted.
\item[B, irreconcilable.] The curve was traced but no admissible IPD was found even
after every relaxation level had been exhausted.
\item[C, partial portfolio.] Fewer than $\Omega=50$ admissible IPD were found before
the run limit of $40\,\Omega$ runs. Relaxation begins only after $R$ runs without any
admissible dataset, so an arm that yields a few admissible IPD keeps its full
constraints and ends with a smaller portfolio.
\item[D, reconstructed after relaxation.] Fifty admissible IPD were found after one
or more relaxation levels.
\item[E, reconstructed under the full criterion.] Fifty admissible IPD were found
with no relaxation.
\end{description}

\begin{table}[!ht]
\centering
\begin{threeparttable}
\caption{The 108 arms of the case study by category.\label{tab:categories}}
\begin{tabular}{@{}llrr@{}}
\toprule
 & Category & Arms & Figures$^{a}$\\
\midrule
A & Not traced & 4 & 1\\
B & Irreconcilable & 0 & 0\\
C & Partial portfolio ($<50$ IPD) & 7 & 7\\
D & Reconstructed after relaxation & 44 & 29\\
E & Reconstructed under the full criterion & 53 & 35\\
\midrule
 & Total & 108 & 53\\
\bottomrule
\end{tabular}
\begin{tablenotes}
\item[$^{a}$] The arms of one figure can fall into different categories.
\end{tablenotes}
\end{threeparttable}
\end{table}

No arm is irreconcilable: every arm whose curve could be traced has at least one
admissible IPD, as Proposition~1 of the main text leads one to expect when the reading is correct.
Categories~A and C deserve a closer look.

\paragraph{Category A: one figure the tracer cannot read.} The four untraced arms
form one figure, the overall survival of NCT02545504 by age group (Figure~4B of
the publication), in which the two arms of each age group share a colour, blue below
$65$ years and red at $65$ and above, and differ only in being solid or dashed. Every
pixel of a colour then belongs to two curves, and any trace of that colour is a
mixture of them; the tracer recognizes the situation from the colours and line styles
reported in stage 1 and declares the four arms untraceable rather than reconstruct a
mixture. The stage-2 readings alone are no substitute: reconstructed from them, the two
experimental arms admitted datasets only after the curve tolerance and the
second-order constraint had been relaxed by several levels, and the two placebo arms
admitted none (Figure \ref{fig:S_untraceable}). Figures that separate their arms only by line style, or that are printed in
grey scale, lie outside what a colour-based tracer can read, and would need either a
tracer that follows a line through its dashes or a manual reading.

\begin{figure}[!ht]
    \centering
    \includegraphics[width=1\linewidth]{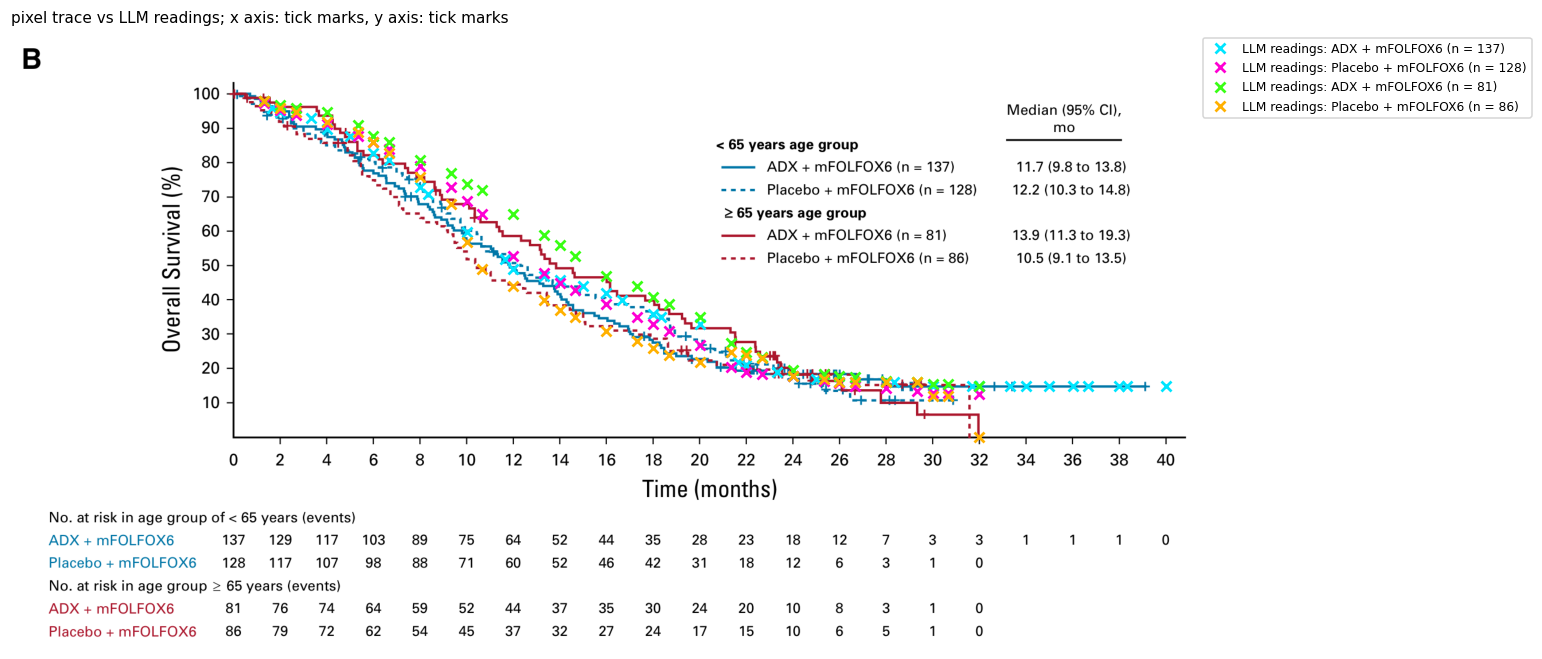}
    \caption{The one figure in the case study that the tracer cannot read: overall
    survival by age group in NCT02545504 (Figure~4B of the publication). The two arms
    of each age group share a colour, blue below 65 years and red at 65 and above, and
    differ only in being solid or dashed, so every pixel of a colour belongs to two
    curves and any trace of that colour would be a mixture of them. The tracer
    recognizes the situation from the colours and line styles reported by the language
    model and declares all four arms untraceable; no pixel trace is drawn. The crosses
    are the language-model readings of each arm, which are not precise enough to stand
    in for the trace: reconstructed from them alone, the two ADX arms admitted datasets
    only after the curve tolerance and the second-order constraint had been relaxed by
    several levels, and the two placebo arms admitted none.}
    \label{fig:S_untraceable}
\end{figure}

\paragraph{Category C: when the printed totals are hard to meet.} The seven
partially reconstructed arms (Table~\ref{tab:partial}) are not failures of the
reading: each has at least one admissible IPD, which by construction reproduces its
risk table, follows its curve and meets every printed summary. They are arms in which
such IPD are rare. In six of them the binding constraint is second-order, most often
the number of events printed beside the figure: a reconstruction can follow the curve
and the risk table and still carry a few events more or fewer than printed, because the
number of events is decided in large part by the censoring drawn in the sparse tail of
the curve, where the risk table gives little guidance. In NCT02107703 the printed median
and its interval bind as often as the number of events. In the seventh arm, the single
arm of NCT02954536, the curve constraint binds in the interval from 24 to 30 months,
where the traced curve and the risk table are hard to reconcile. A smaller portfolio
is a weaker description of the uncertainty the figure leaves, but not a less faithful
one: every IPD in it meets the full criterion of Section~2 of the main text.

\begin{table}[!ht]
\centering
\begin{threeparttable}
\caption{The seven arms with a partial portfolio.\label{tab:partial}}
\begin{tabular}{@{}llrl@{}}
\toprule
Figure & Arm & IPD & Most frequent failure\\
\midrule
NCT02107703, Fig.~1B & Abemaciclib $+$ fulvestrant & 1 & median and interval; events\\
NCT02954536, Fig.~2B & Single arm & 49 & curve, 24--30 months$^{a}$\\
NCT03615326, Fig.~1A & Pembrolizumab & 7 & events (267 printed)\\
NCT03653507, Fig.~3A & Placebo $+$ CAPOX & 3 & events (174 printed)\\
NCT04639986, Fig.~2C & Chemotherapy & 42 & events (89 printed)\\
NCT04662710, Fig.~3A & Chemotherapy & 26 & events (306 printed)\\
NCT05008783, Fig.~ED3B & Placebo & 39 & events (85 printed)\\
\bottomrule
\end{tabular}
\begin{tablenotes}
\item[$^{a}$] After the curve tolerance $\delta_s$ was relaxed from $0.02$ to $0.03$.
\end{tablenotes}
\end{threeparttable}
\end{table}

\section{Uncertainty and identification}\label{sec:S5}
This section states what a published figure determines about the data behind it. It first
defines the uncertainty of a reconstruction, then treats the idealized case of a figure read
without error, $\delta_t=\delta_s=0$, then the admissible set under real digitization, and
last the censoring marks.

\subsection*{S5.1 Definitions}
A dataset of one arm is a multiset $x=\{(y_i,\Delta_i)\}_{i=1}^{n^*}$ with $y_i>0$ and
$\Delta_i\in\{0,1\}$; patients are exchangeable, so two datasets that differ by a permutation
of patients are the same. Let $\mathcal X$ be the set of all such datasets. For $x\in\mathcal X$
let $u_1<\dots<u_m$ be its distinct event times,
$d_\ell(x)=\#\{i:y_i=u_\ell,\Delta_i=1\}$ be the number of events at time $\ell$, and $N_x(t)=\#\{i:y_i\ge t\}$ the number at risk at
$t$ and $n_\ell(x)=N_x(u_\ell)$. A patient censored at an event time is counted in risk the risk set at
that time. The Kaplan--Meier estimate is
$S_x(t)=\prod_{\ell:u_\ell\le t}\{1-d_\ell(x)/n_\ell(x)\}$. All observations are assumed to lie
in $[0,t_{\mathrm{end}}]$, where $t_{\mathrm{end}}$ is the end of the drawn curve.

A figure read without error yields the observation
\[
O(x)=\Bigl(S_x \text{ on } [0,t_{\mathrm{end}}],\; \bigl(N_x(t_k)\bigr)_{k=0}^{K},\; \Sigma(x)\Bigr),
\]
where $\Sigma(x)$ collects the printed summaries computed from $x$: the median, its
Brookmeyer--Crowley interval with Greenwood's variance, and the number of events
$\sum_\ell d_\ell(x)$. A publication printing $o$ is \emph{consistent} if $o=O(x^0)$ for the
original data $x^0$.

\begin{definition}[Admissible set, uncertainty, identification]\label{def:uncertainty}
Under perfect digitization the admissible set of a published observation $o$ is
$\mathcal A_0(o)=\{x\in\mathcal X:O(x)=o\}$; under real digitization, $\mathcal A_\delta$ is the
set of datasets satisfying constraints (1)--(3). For a functional
$\theta:\mathcal X\to\mathbb R$ and an admissible set $\mathcal A$, the
\emph{reconstruction uncertainty} of $\theta$ is $U(\theta)=\{\theta(x):x\in\mathcal A\}$.
The functional is \emph{identified} if $U(\theta)$ is a single point, and the data are
identified if $\mathcal A=\{x^0\}$.
\end{definition}

A portfolio $\tilde D_1,\dots,\tilde D_\Omega$ drawn from $\mathcal A$ estimates $U(\theta)$ from
the inside by $\{\theta(\tilde D_1),\dots,\theta(\tilde D_\Omega)\}$; this is the spread
reported in the main text, where $\Omega=50$.

\subsection*{S5.2 When the admissible set is empty}
If the publication is consistent, $x^0\in\mathcal A_0(o)$, so $\mathcal A_0(o)\neq\emptyset$;
under real digitization, Proposition~1 of the main text gives the same conclusion whenever every
reading lies within its tolerance. Consequently an empty admissible set obtained from an exact
reading means that $o$ is not in the image of $O$: the printed quantities were not all computed
from one dataset by the conventions above. Examples are summaries taken at another data cut-off,
an event count that includes events after $t_{\mathrm{end}}$, an interval computed by another
method, or numbers at risk counted with another convention.

\subsection*{S5.3 Identification under perfect digitization}
Let $\mathcal P=\{u_1,\dots,u_m\}\cup\{t_0,\dots,t_K\}\cup\{t_{\mathrm{end}}\}$ with elements
$0=\pi_0<\pi_1<\dots<\pi_J=t_{\mathrm{end}}$, and call $C_j=[\pi_j,\pi_{j+1})$,
$j=0,\dots,J-2$, and $C_{J-1}=[\pi_{J-1},t_{\mathrm{end}}]$ the \emph{cells}. For each event time write the relative height of the step,
$h_\ell=1-S(u_\ell)/S(u_\ell^-)=d_\ell/n_\ell$, and let $m_\ell=n_{\ell-1}-d_{\ell-1}$
($m_1=n^*$) bound $n_\ell$ from above.

\noindent\textbf{(A1)} For every $\ell$, $(d_\ell,n_\ell)$ is the only pair of integers with
$d/n=h_\ell$, $d\ge1$ and $n\le m_\ell$ that is compatible with the risk table. This holds in
particular if events are known not to be tied ($d_\ell=1$), or if $h_\ell=p_\ell/q_\ell$ in
lowest terms with $q_\ell>m_\ell/2$.

\begin{proposition}[What the figure identifies]\label{prop:identify}
Let the publication be consistent and (A1) hold. Then $x\in\mathcal A_0(o)$ if and only if
(a) $x$ has the events of $x^0$, that is the same $(u_\ell,d_\ell)$ for every $\ell$, and
(b) for every cell $C_j$, $x$ and $x^0$ have the same number of censored observations in $C_j$.
\end{proposition}

\begin{proof}
If $O(x)=O(x^0)$, the steps of $S_x$ give the event times and the ratios $h_\ell$, and (A1)
gives $(d_\ell,n_\ell)$, hence (a). Then $N_x$ is known at every point of $\mathcal P$ below
$t_{\mathrm{end}}$: $N_x(u_\ell)=n_\ell$ and $N_x(t_k)=r_k$. For $j\le J-2$ the number of
censored observations in $C_j$ equals $N_x(\pi_j)-N_x(\pi_{j+1})$ minus the events at $\pi_j$;
in the last cell it equals $N_x(\pi_{J-1})$ minus the events at $\pi_{J-1}$, since no
observation exceeds $t_{\mathrm{end}}$. These numbers are therefore the same for $x$ and $x^0$,
hence (b). Conversely, (a) and (b) imply $N_x(\pi)=N_{x^0}(\pi)$ for every
$\pi\in\mathcal P$ below $t_{\mathrm{end}}$, so $n_\ell(x)=n_\ell(x^0)$, the
Kaplan--Meier estimates agree, the numbers at risk at the $t_k$ agree, and $\Sigma(x)=\Sigma(x^0)$
because the median, the Greenwood variance and the event count depend on $x$ only through
$\{(u_\ell,d_\ell,n_\ell)\}$.
\end{proof}

\begin{corollary}[Uniqueness]\label{cor:unique}
Under the assumptions of Proposition~\ref{prop:identify}, $\mathcal A_0(o)=\{x^0\}$ if and only
if $x^0$ has no censored observation. Otherwise $\mathcal A_0(o)$ is uncountable: each censored
observation may take any value in its cell, every cell having positive length.
\end{corollary}

\begin{corollary}[Identified functionals]\label{cor:identified}
Under the assumptions of Proposition~\ref{prop:identify}, every functional that depends on $x$
only through $\{(u_\ell,d_\ell,n_\ell)\}$ is identified: the Kaplan--Meier estimate at any
time, landmark survival, the restricted mean survival time up to $\tau\le t_{\mathrm{end}}$,
the median, its Brookmeyer--Crowley interval and the number of events. In particular, under
perfect digitization constraints (1) and (2) imply constraint (3), and the second-order
constraints bind only because real digitization is imperfect. Functionals that depend on the
censoring times, such as the reverse Kaplan--Meier median follow-up, are in general not
identified when $x^0$ has censored observations.
\end{corollary}

\begin{remark}[Comparisons between arms]\label{rem:twoarms}
For two arms $a$ and $b$ the admissible set is the product
$\mathcal A_0(o^a)\times\mathcal A_0(o^b)$. The log-rank statistic and the Cox partial likelihood
use, at each event time $u$ of arm $a$, the number at risk $N_{x^b}(u)$ in arm $b$. By
Proposition~\ref{prop:identify}, $N_{x^b}(u)$ is identified only if the cell of arm $b$ that
contains $u$ has no censored observation before $u$ in every admissible dataset; otherwise it
varies over $\mathcal A_0(o^b)$. The hazard ratio is therefore not identified in general, even
under perfect digitization; its uncertainty $U(\mathrm{HR})$ comes from the position of the
censorings of one arm relative to the events of the other.
\end{remark}

\subsection*{S5.4 The admissible set under real digitization}
With $\delta_t,\delta_s>0$ the counts are no longer determined by the reading. Call a
\emph{configuration} $\kappa$ the following discrete data: the number of events and of
censored observations in every fine interval $F_j$, their order within it, the events at which
the reconstructed curve and its Brookmeyer--Crowley limits first reach one half, and the order
of the reconstructed interval limits relative to $M^*_L$ and $M^*_U$. A configuration fixes
the Kaplan--Meier heights at every read time, the numbers at risk at the $t_k$ and the event
count; the remaining free quantities are the times. Given $\kappa$, constraints (1)--(3) are
linear inequalities in the times: the event windows $W_j$, the fine intervals and the order
within them, $|\tilde M-M^*|\le\varepsilon_t$, and the overlap condition, which for a fixed
order of the limits reads $|\cap|-\rho|\cup|\ge0$ with both lengths linear in the limits.
Hence, with patients matched by rank within each fine interval and status,
\begin{equation}\label{eq:union}
\mathcal A_\delta=\bigcup_{\kappa\in\mathcal K}\mathcal P_\kappa,
\end{equation}
where $\mathcal K$ is finite and every $\mathcal P_\kappa$ is a convex polyhedron. Under
perfect digitization and (A1), $\mathcal K$ has a single element and $\mathcal A_0$ is convex.
Three consequences describe how further admissible datasets follow from a few.
\begin{enumerate}
\item \emph{Averaging within a configuration.} If $x,x'\in\mathcal P_\kappa$ and
$\lambda\in[0,1]$, the dataset whose times are $\lambda y_i+(1-\lambda)y'_i$, with patients
matched as above, lies in $\mathcal P_\kappa$. Two admissible datasets of one configuration
thus yield infinitely many, which differ only in the positions of events and censorings
within their windows and cells; across configurations an average would have fractional
counts and is not a dataset.
\item \emph{Recombination across risk-table intervals.} Every admissible dataset has
$N(t_k)=r_k$ for all $k$, so the observations with $t_{k-1}\le y<t_k$ number
$r_{k-1}-r_k$ in each of them. Given admissible $x^{(1)},\dots,x^{(n)}$ and indices
$\ell_1,\dots,\ell_K\in\{1,\dots,n\}$, the dataset that takes these observations from
$x^{(\ell_k)}$ for each $k$ satisfies constraint (1). Constraints (2) and (3) must be checked
again, because the reconstructed curve at the start of an interval is a product over the
intervals before it, and the median, its interval and the event count depend on the whole
dataset. At most $n^K$ candidates arise in this way.
\item \emph{Finiteness.} Since $\mathcal K$ is finite, the admissible datasets that differ in
their counts are finitely many, while $\mathcal A_\delta$ itself is uncountable whenever some
$\mathcal P_\kappa$ has interior. The portfolio size $\Omega$ is the number of datasets
reported, not a bound on $\mathcal A_\delta$.
\end{enumerate}

\subsection*{S5.5 Censoring marks}
Corollary~\ref{cor:unique} locates the whole uncertainty of a perfectly digitized figure in the
censoring times, which many figures draw as tick marks.

\begin{proposition}[Censoring read]\label{prop:marks}
Under the assumptions of Proposition~\ref{prop:identify}, suppose in addition that the multiset
of censoring times $\{y_i:\Delta_i=0\}$ of $x^0$ is observed. Then $\mathcal A_0=\{x^0\}$.
\end{proposition}

\begin{proof}
By Proposition~\ref{prop:identify}(a) the events of every admissible dataset are those of
$x^0$, and its censored observations are now given.
\end{proof}

A partial reading also reduces the uncertainty. If marks are resolved only into counts within
sub-intervals, as when adjacent marks merge, the cells $C_j$ are refined by those sub-intervals
and Proposition~\ref{prop:identify} holds with the refined cells; $U(\theta)$ then shrinks as
the partition refines, and only the unresolved clusters remain free. With the censoring times
given, the search of Step~2 of the main text over $\tilde c_k$ is replaced by a deterministic
inversion of the curve. The marks are not used here because current vision-language models
cannot read them reliably (Discussion of the main text).